\documentclass[11pt]{article}

\makeatletter
\renewcommand\section{\@startsection {section}{1}{\z@}%
                                   {-3.5ex \@plus -1ex \@minus -.2ex}%
                                   {2.3ex \@plus.2ex}%
                                   {\normalfont\bfseries}}

\def\@maketitle{%
  \newpage
  \null
  \vskip 2em%
  \begin{center}%
  \let \footnote \thanks
    {\large \@title \par}%
    \vskip 1.5em%
    {\normalsize
      \lineskip .5em%
      \begin{tabular}[t]{c}%
        \@author
      \end{tabular}\par}%
    \vskip 1em%
  \end{center}%
  \par
  \vskip 1.5em}
\makeatother
\usepackage{graphicx}

\usepackage{amsmath,amsfonts,amscd,bezier,amssymb}
\usepackage{amssymb}
\usepackage[normalem]{ulem}
\usepackage[all]{xy}

\usepackage[latin1]{inputenc}

\usepackage{color}

\newenvironment{proof}{\noindent\textbf{Proof: }}{\hfill \small $\Box$}

\newcommand{\prt}[1]{\langle #1 \rangle}

\newcommand{\raiz}[2]{\root #1 \of {#2}}

\newtheorem{theorem}{Theorem}
\newtheorem{lemma}[theorem]{Lemma}
\newtheorem{proposition}[theorem]{Proposition}
\newtheorem{corollary}[theorem]{Corollary}
\newtheorem{definition}[theorem]{Definition}
 \newtheorem{example}{Example}
 \newtheorem{remark}{Remark}

\begin{document}

\title{\bf Generalized Interval-valued OWA Operators with Interval Weights Derived from Interval-valued Overlap Functions}


\author{{\bf Benjamin Bedregal$^a$ and Humberto Bustince$^b$ and E. Palmeira$^c$ } \\
{\bf  Gra\c caliz Pereira Dimuro$^d$ and Javier Fernandez$^b$} \\ 
\small $^a$ Departamento de Informática e Matem\'atica Aplicada, \\ 
\small Universidade Federal do Rio Grande do Norte \\ 
\small \textit{bedregal\@ dimap.ufrn.br} \\
\small $^b$ Departamento de Autom\'atica y Computaci\'on,  \\ 
\small Universidad P\'ublica de Navarra\\ 
\small \textit{bustince\@ unavarra.es} \\
\small $^c$ Departamento de Ci\^encias Exatas e Tecnol\'ogicas \\
\small  Universidade Estadual de Santa Cruz, \\ 
\small		\textit{espalmeira\@ uesc.br} \\
\small $^d$ Centro de Ci\^encias Computacionais, \\ 
\small Universidade Federal do Rio Grande\\
\small 		\textit{gracalizdimuro\@ furg.br;gracaliz\@ gmail.com}
}
		

%

\maketitle

\begin{abstract} In this work   we extend to the interval-valued setting the notion of an overlap functions and we discuss a method which makes use of interval-valued overlap functions for constructing OWA operators  with interval-valued weights. . Some properties of interval-valued overlap functions and the derived interval-valued OWA operators are analysed. We specially focus on the homogeneity and migrativity properties.
\end{abstract}
\vspace{12pt}\noindent{\bf Keywords} Interval-valued fuzzy sets \and interval-valued overlap functions \and Interval-valued overlap OWA operators \and interval weighted vector \and  migrativity \and 
homogeneity

\section{Introduction}

Interval-valued fuzzy sets \cite{Zad65} have been succesfully applied in many different problems. Just to mention some of the most recent ones, interval-valued fuzzy sets have been used in decision making (see, e.g., the works by Khalil and Hassan \cite{Khalil2016} and  Cheng et al. \cite{Cheng2016}), image processing (see, e.g., the works by Barrenechea et al. \cite{Barrenechea}, Pagola et al. \cite{Pagola} and Melin et al. \cite{Melin}) or classification (see, e.g., the works by Sanz et al. \cite{Sanz1,Sanz2}). They have also been the origin of rich theoretical studies, as, for instance,  the works by Bedregal et al. \cite{BDSR10,DBLP:conf/eusflat/BedregalDR09}, Dimuro et al. \cite{DBSR11}, Reiser et al. \cite{springerlink:10.1007/978-3-540-73445-1_22}  and the recent works by Zywica et al.  \cite{Zywica2016105} and Takác \cite{Zdenko}.

From the point of view of applications, interval-valued fuzzy sets are a suitable tool to represent uncertain or incomplete information. In particular, the length of the interval-valued membership degree of a given element can be understood as a measure of the lack of certainty of the expert for providing an exact membership value to that element \cite{NKZ97}. Interval degrees
are also be used to summarize the opinions of several experts. In this case,  the left and right interval endpoints can be, for instance, the least and the greatest membership degrees
provided by a group of experts. This makes interval-valued fuzzy sets very useful for multiexpert decision making problems, when the experts are asked to express 
numerically their preferences on several alternatives, as discussed by Bustince et al. \cite{bgbkm} (see also the discussions about that in 
\cite{BDSR10,BS13,bustince-etal-1,bustincesolo}).

Besides, another relevant tool for many different application is that of OWA operators, introduced by Yager \cite{Ya88} and largely used in the literature (see, e.g.: \cite{Jin2016,Jin2016-2,Mesiar2}). Its usefulness has led to the consideration of different possible extensions for Atanassov intuitionistic fuzzy sets (\cite{LWC13,mitchel,xu-yager,Yag09}) and for interval-valued fuzzy
sets (\cite{bgbkm,ChiZho2013,XuDa,ZCJG2008}).

In the latter case, however, one of the key problems is how to build and normalize interval-valued weights.
In  the literature, interval-valued weights are used in several contexts, in order to face the problem  of real-world applications in which there are a lot of uncertainty involved and  lack of consensus among the modeling experts.  Pavlacka \cite{Pavla} presented a review of the existing methods for normalization of interval weights. For example, in the context of multi-criterion decision making,    Wang and Li \cite{Wang} used a hierarchical structure to aggregate local interval weights into global
interval weights, by means of a pair of linear programming
models to maximize the lower and upper bounds of the
aggregated interval value.

However, in the definitions of interval-valued OWA operators found in the literature, the weighted vector is composed, in general, by real numbers. Due to this limitation of the actual models of interval-valued OWA operators, in this paper, we propose the use of  interval weights. In order to define these weights we propose the extension of the so-called overlap functions \cite{BDH13,Bus10a,Bus10b,DB14,DB15} to the interval-valued setting. In this way, the normalization method proposed here makes use of the properties of   aggregation functions, and, thus, it is defined in flexible terms.


%
%

Then, the objectives of this paper are:
\begin{itemize}

\item To introduce the concept of interval-valued overlap functions, and to analyze some of its most relevant properties, such as migrativity and homogeneity;

\item To define the normalization of an interval-valued weighted vector by means of a general  aggregation function,  and to determine which conditions normalized weighted vectors should fulfill;

\item To develop a construction method of  interval-valued OWA operators based on interval-valued overlap functions, considering  interval-valued weights;

    \item To study the properties of such OWA operators, specially considering the migrativity and homogeneity of interval-valued overlap functions.

\end{itemize}


This work is organized as follows. In Section \ref{sect1}, we recall some basic concepts that we are going to use along the paper. Next, we define the basic order relations between intervals and, in Section 4, we define the concept of interval-valued overlap function and we study some of its most important properties. In Section 5 we present the concept of normalized weighted vector, and we analyze the definition of interval-valued OWA operators with interval weights. We also study the conditions that the functions used for the definition must fulfill to recover idempotency and other properties. We finish with conclusions and references.

\section{Preliminary Concepts}\label{sect1}
 We start, in this section, recalling some well-known concepts which will be ncessary for our subsequent developments.

 Consider a function $f: [0,1]^n \rightarrow [0,1]$. Given $i\in \{1,\ldots,n\}$, we say that the component $i$ is necessary if it does not exist a function \[g:\underbrace{[0,1]\times \ldots\times [0,1]}_{i-1} \times \underbrace{[0,1]\times \ldots \times [0,1]}_{n-i}\rightarrow [0,1]\] such that
\[f(x_1,\ldots,x_n)=g(x_1,\ldots,x_{i-1},x_{i+1},\ldots,x_n),\] for any $(x_1,\ldots,x_n)\in [0,1]^n$.

\subsection{Aggregation functions}

A crucial concept for the present paper is that of aggregation function (see \cite{Calvo2}).

\begin{definition}
A n-ary aggregation function is a mapping $M:[0,1]^n \to [0,1]$ such that

\noindent $(M1)$ $M(0,\dots,0)=0$ and $M(1,\dots,1)=1$;

\noindent $(M2)$ $M$ is increasing in each argument: for every $i=1,\dots,n$, if $x_i \le y_i$  then $M(x_1,\dots,x_n)\le M(y_1,\dots,y_n)$.

\end{definition}

Several other properties can be required for aggregation functions. In particular, in this work we are interested in the following two ones.
\begin{description}
\item $(M3)$ If $M(x_1,\ldots,x_n)=0$  then there is $i=1,\ldots,n $ such that $x_i=0$;

\item $(M4)$ If $M(x_1,\ldots,x_n)=1$  then there is $i=1,\ldots,n $ such that $x_i=1$.
\end{description}

Among the class of aggregation functions, the so-called OWA operators are a very relevant case.  These operators were defined by Yager in \cite{Ya88,Ya96} in the following way:

\begin{definition}
Let $w=(w_1,\dots,w_n) \in [0,1]^n$ be a weighted vector (i.e., $w_i\in [0,1]$ and $\sum\limits_{i=1}\limits^n w_i=1$). An OWA operator of dimension $n$ associated to the weighted vector $w$ is a function $OWA:[0,1]^n \to [0,1]$ defined by
\[
OWA(x_1,\dots,x_n)=\sum _{i=1}^n w_i x_{(i)}
\]
where $(.)$ denotes a permutation of $\{1,\dots,n\}$ such that $x_{(1)} \ge x_{(2)} \ge \dots \ge x_{(n)}$.
\end{definition}

Another relevant example of aggregation functions is provided by overlap functions \cite{Bus10a}.

\begin{definition}\label{gsol} A function $G_O:[0,1]^2\rightarrow [0,1]$ is an overlap function if it satisfies the following conditions:

\noindent $(GO1)$ $G_O$  is commutative;

\noindent $(GO2)$ $G_O(x,y)=0$ if and only if $xy =0$;

\noindent $(GO3)$ $G_O(x,y)=1$ if and only if $xy = 1$;

\noindent $(GO4)$ $G_O$ is non-decreasing;

\noindent $(GO5)$ $G_O$ is continuous.

\end{definition}

Several well-known functions fall into the scope of overlap functions, as, for instance, continuous t-norm without zero divisors, see \cite{BDH13,DB14,DB15,DB15a}.

It is worth to mention that if an overlap function has a neutral element, then by (GO3) it is necessarily 1. Moreover, associative overlaps always have 1 as neutral
element and so they are continuous t-norms without zero divisors \cite{Bus10a}.

%
%

\begin{example} Nevertheless,  there are overlaps having 1 as neutral element such that  they are not associative.
For example, $$G_O(x,y)=\min(x,y)\max(x^2,y^2)$$ or, more generally, $$G_O(x,y)=\min(x,y)\max (x^p,y^p)$$
for every $p>0$.
\end{example}

%
%
%

We denote as $\mathcal{O}$ the set of all overlap functions. It follows that
$\prt{\mathcal{O},\leq_{\mathcal{O}}}$, where $\leq_{\mathcal{O}}$ is defined for $G_{O_1},G_{O_2} \in \mathcal{O}$ by
\[
G_{O_1} \leq_{\mathcal{O}} G_{O_2} \mbox{ if and only if } G_{O_1} (x,y) \leq G_{O_2}(x,y)
\]
for all $x,y \in [0,1]$, is a lattice \cite[Theorem 3]{Bus10a}. In particular, the supremum and infimum of two arbitrary overlap functions $G_{O_1}$ and $G_{O_2}$
 are again overlap functions

\begin{equation} \label{eq-sup-inf-over}
\aligned &G_{O_1}\vee G_{O_2}(x,y)=\max(G_{O_1}(x,y),G_{O_2}(x,y))\mbox{ and  }\\&G_{O_1}\wedge G_{O_2}(x,y)=\min(G_{O_1}(x,y),G_{O_2}(x,y))\endaligned
\end{equation}

%

The class of overlap functions is also convex, as was proved in \cite{Bus10a}. That is, if we take two overlap functions $G_{O_1}$ and $G_{O_2}$.
then for each $w_1,w_2\in[0,1]$ such that $w_1+w_2=1$ we have that their convex sum

\begin{equation}
 G_O(x,y)=w_1G_{O_1}(x,y)+w_2G_{O_2}(x,y)
\end{equation}
is also an overlap \cite[Corollary 1]{Bus10a}.

%

\subsection{Interval-valued fuzzy sets}

We denote by $L([0,1])=\{[a,b]\in [0,1]^2 : a\leq b\}$ be the set of all closed subintervals of the unit interval $[0,1]$.    In order to simplify the notation, we  use  the projections of an interval $[a,b] \in  L([0,1])$, given by:
\begin{equation}\label{eq-under-over}
 \underline{[a,b]}=a\mbox{ and }\overline{[a,b]}=b.
\end{equation}

Those intervals $X$ such that $\underline{X}=\overline{X}$ are called degenerate intervals or diagonals elements of $L([0,1])$. We denote by $\mathcal{D}$ the set of all degenerate intervals.

Recall that an interval-valued fuzzy set (IVFS) $\mathcal{A}$ on a universe $U$
is defined by means of an interval membership function  $\mu_\mathcal{A}:U\rightarrow L([0,1])$.

Let's consider now the problem of ordering in $L([0,1])$. First of all, we have the following partial orders:
\begin{description}
  \item[Product order:] $X\leq_{Pr} Y\mbox{ iff } \underline{X}\leq \underline{Y}\mbox{ and } \overline{X}\leq \overline{Y}$; and
\item[Inclusion order:] $X\subseteq Y\mbox{ iff } \underline{Y}\leq \underline{X}\leq\overline{X}\leq \overline{Y}$
 \end{description}

$(L([0,1]),\leq_{Pr})$ is a continuous lattice  and so  it is a bounded lattice (\cite{GHK03,Sco71}). In fact the bottom of this lattice is $[0,0]$, its top is $[1,1]$, the supremum and infimum of
two arbitrary intervals $X,Y\in L([0,1])$ are the following intervals:
$$\begin{array}{ll}
 X\vee Y & =[\max(\underline{X},\underline{Y}),\max(\overline{X},\overline{Y})],\mbox{ and }\\
X\wedge Y & =[\min(\underline{X},\underline{Y}),\min(\overline{X},\overline{Y})].
\end{array}$$


Notice that, in general, it is possible to define a linear order in $L([0,1])$ as follows:

\begin{definition}\cite{BusFerKolMes,bgbkm}
A binary relation $\preceq$ on $L([0,1])$  is an admissible order if it is a linear order on $L([0,1])$ refining $\leq_{Pr}$.
\end{definition}



Admissible orders generated by aggregation functions extend the usual product order on the set of intervals. A deep analysis of this kind of orders can be found in \cite{BusFerKolMes}.

Some operations on $L([0,1])$ that are  used in this paper are defined, for all $X,Y\in L([0,1])$, as:
\begin{description}
 \item[Product:] $XY=[\underline{X} \underline{Y},\overline{X}\overline{Y}]$;
\item[Exponentiation 1:] $X^{[k_1,k_2]}=[\underline{X}^{k_2},\overline{X}^{k_1}]$ for  $0<k_1\leq k_2$;
\item[Exponentiation 2:] {$X^{-[k_1,k_2]}=[\overline{X}^{-k_1},\underline{X}^{-k_2}]$ for any $X\in L([0,1])$ and $0<k_1\leq k_2$;}

\item[Complement:] $X^c=[1-\overline{X},1-\underline{X}]$; and
\item[Arctan:] $Arctan(X)=[\arctan (\underline{X}),\arctan (\overline{X})]$.
\end{description}

When $k_1=k_2$ usually we will write $X^{k_1}$ instead of $X^{[k_1,k_2]}$.

We discuss now the relation between interval-valued and real-valued functions.

A function $F: L([0,1])^n\rightarrow L([0,1])$ is an interval representation of a function
$f:[0,1]^n\rightarrow [0,1]$ (\cite{BT06,DBSR11,SBA06}), if  for any
$X_1,\ldots,X_n\in L([0,1])$, $f(x_1,\ldots,x_n)\in F(X_1,\ldots,X_n)$ when $x_i\in X_i$ for $i=1,\ldots, n$.

Notice that the interval representation of a given function $f$ is not unique, in general. In fact, the function $F: L([0,1])^n\rightarrow L([0,1])$ defined as
$F(X_1,\ldots,X_n)=[0,1]$, for every $X_1,\ldots,X_n \in L([0,1])$, is an interval representation of any function $f:[0,1]^n\rightarrow [0,1]$.

If $F$ is an interval representation of some real function then it is inclusion
monotonic, that is, $F(X_1,\ldots,X_n)$ $\subseteq F(Y_1,\ldots,Y_2)$ whenever
$X_i\subseteq Y_i$, for each $i=1,\ldots,n$.

The extension of the notion of an aggregation function to the interval-valued setting can be done in the following way.

\begin{definition}
Let $\le$ be an order that extends $\le _{Pr}$.\footnote{Observe that $\le$ does need to be a linear order.} An interval-valued aggregation function with respect to $\le$ is a function $M:L([0,1])^n \to L([0,1])$ such that
\begin{enumerate}
\item $M$ is increasing in each argument with respect to $\le$: for every $i=1,\dots,n$, if $X_i \le Y_i$  then $M(X_1,\dots,X_n)$ $\le M(Y_1,\dots,Y_n)$;
\item $M([0,0],\dots,[0,0])=[0,0]$ and $M([1,1],\dots,[1,1])=[1,1]$.
\end{enumerate}
\end{definition}

We finish recalling two notions which are of interest when dealing with overlap functions: migrativity and homogeneity.

The concept of $\alpha$-migrativity was introduced by Durante {\em et al.} in \cite{Durante}, see also \cite[Problem 1.8(b)]{Mesiar} and \cite{Fodor2}.  Santana et al. \cite{SBS14} extended, in a natural way, the notion of migrativity to the interval-valued setting.

An interval-valued function $F: L([0,1])^2\rightarrow L([0,1])$ is migrative if for any $\alpha,X,Y\in L([0,1])$, we have that
 $F(\alpha X,Y)=F(X,\alpha Y)$.

The following results are analogous to  those discussed  in \cite{Bus09}, so we do not include a proof.

\begin{lemma}\label{lem-IV-mig1}
 A function  $F: L([0,1])^2\rightarrow L([0,1])$ is migrative if and only if $F(X,Y)=F([1,1],XY)$.
\end{lemma}

\begin{proposition}\label{pro-IV-mig1}  \cite[Theorem 1.1]{SBS14}
 A function  $F:L([0,1])^2\rightarrow L([0,1])$ is migrative if and only if the interval function $g_F:L([0,1])\rightarrow L([0,1])$, defined, for all $X\in L([0,1])$, by
\begin{equation}\label{eq-g-F}
g_F(X)=F([1,1],X)
\end{equation}
is such that
 $F(X,Y)=g_F(XY)$, for each $X,Y\in L([0,1])$.
\end{proposition}


\begin{corollary}
If  $F: L([0,1])^2\rightarrow L([0,1])$ is migrative, then, for any $X\in L([0,1])$, it holds that $F([1,1],X)=F(\sqrt{X},\sqrt{X})$.
\end{corollary}

With respect to homogeneity, it can also be extended to the interval-valued setting in the following way. An interval function $F:L([0,1])^n\rightarrow L([0,1])$ is homogeneous of order $K=[k_1,k_2]$, with $0< k_1\leq k_2$, if, for any
$\alpha,X_i\in L([0,1])$ with $i=1,\ldots,n$, the identity

$$F(\alpha X_1,\ldots, \alpha X_n)=\alpha^K F(X_1,\ldots,X_n)$$
holds.


\section{Interval-Valued Overlaps}
 In this section, we introduce the concept of interval-valued overlap functions, which is the key concept of this work.

\begin{definition}\label{def-int-overlap} A function $O:L([0,1])^2\rightarrow L([0,1])$ is an interval-valued overlap function if it satisfies the following conditions:

\noindent \begin{description}
 \item{$(O1)$.-}  $O$  is commutative;

\item $(O2)$.- $O (X,Y)=[0,0]$ if and only if $XY =[0,0]$;

\item $(O3)$.- $O(X,Y)=[1,1]$ if and only if $XY= [1,1]$;

\item $(O4)$.- $O$ is {monotonic in the second component, i.e. $O(X,Y)\leq_{Pr} O(X,Z)$ when $Y\leq_{Pr} Z$}.

\item $(O5)$.- $O$ is Moore continuous \cite{AB97,Moo79,SBA06}.
\end{description}

\end{definition}

{Note that, by $(O1)$ and $(O4)$, interval-valued overlaps also are monotonic in the first component.} Observe also that the first four points in our definition are analogous to the first four points in Definition~\ref{gsol}. In the last point, however, and in order to have a notion of continuity, we take Moore continuity.


Let $\mathfrak{O}$  be the set of all interval-valued overlap  functions. We may define on $\mathfrak{O}$ the binary relation:

$$O_1\leq_{\mathfrak{O}} O_2\mbox{ iff } O_1(X,Y)\leq_{Pr} O_2(X,Y),\mbox{ for all }X,Y{\in} L([0,1])$$
Clearly, $\leq_{\mathfrak{O}}$ is a partial order on $\mathfrak{O}$. Furthermore, and analogously to the case of real-valued overlap functions, we have the following result.

\begin{proposition}
 $(\mathfrak{O},\leq_\mathfrak{O})$ is an unbounded lattice.
\end{proposition}
\begin{proof}
 Let $O_1$ and $O_2$ be interval-valued overlap functions. Then the functions $O_1\vee_\mathfrak{O}O_2,
O_1\wedge_\mathfrak{O}O_2:L([0,1])^2\rightarrow L([0,1])$
defined by
$$O_1\vee_\mathfrak{O}O_2(X,Y)=O_1(X,Y)\vee O_2(X,Y)$$
and
$$O_1\wedge_\mathfrak{O}O_2(X,Y)=O_1(X,Y)\wedge O_2(X,Y)$$
are clearly the supremum and infimum of $O_1$ and $O_2$. It is not hard to prove that both  are also interval-valued overlap functions.

On the other hand, in order to prove that  $(\mathfrak{O},\leq_\mathfrak{O})$ is unbounded it is enough to note that for any interval-valued overlap function
$O$ {and natural number $n\geq 2$,}
the functions $O^{\raiz{n}{\;\;}},O^n:L([0,1])^2\rightarrow L([0,1])$, defined by
$$O^{\raiz{n}{\;\;}}(X,Y)=O(\raiz{n}{X},\raiz{n}{Y})$$ and $$O^n(X,Y)=O(X^n,Y^n),$$ respectively,
also are interval-valued  overlap functions and $O^n <_{\mathfrak{O}} O <_\mathfrak{O} O^{\raiz{n}{\;\;}}$. Therefore, there is neither
 a least nor a great
interval-valued overlap function.
\end{proof}

{In fact, in the above proposition, if we denote by $O^{\raiz{\infty}{\;\;}}=\lim\limits_{n\rightarrow \infty} O^{\raiz{n}{\;\;}}$ and
$O^\infty=\lim\limits_{n\rightarrow \infty} O^n$, then it holds that that}
$${O^{\raiz{\infty}{\;\;}}(X,Y)=}\left\{ \begin{array}{ll}
                                  [0,0] & \mbox{ if $X\vee Y=[0,0]$} \\

                                  [1,1] & \mbox{ otherwise}
                                 \end{array}
                            \right. $$
and
$${O^\infty(X,Y)=}\left\{\begin{array}{ll}
                                  [1,1] & \mbox{ if }X\wedge Y=[1,1] \\

                                  [0,0] & \mbox{ otherwise.}
                                 \end{array}
                            \right.  $$
Observe that these functions are not interval-valued overlap functions, since they do not fulfill the Properties $(O2)$ and $(O3)$,  respectively.

The relation between overlap functions and t-norms in the real case is preserved in the interval-valued framework.

\begin{proposition}
       Let $O$ be an interval-valued overlap function. If  $O$ is associative then $O$ is a
       continuous and positive interval-valued t-norm.
      \end{proposition}
\begin{proof}
Let $g_O:L([0,1])\rightarrow L([0,1])$  be the function $g_O(X)=O(X,[1,1])$. By continuity of $O$, $g_O$
is also continuous. Since, $g_O([0,0])=[0,0]$ and $g_O([1,1])=[1,1]$ then, for any $X\in L([0,1])$,
there exists $Y\in L([0,1])$ such that $g_O(Y)=X$. Therefore, by the associativity  property  and $(O3)$,  we have that
$$\begin{array}{ll}
X& =O(Y,[1,1]) \\
& =O(Y,O([1,1],[1,1])) \\
& =O(O(Y,[1,1]),[1,1]) \\
& =O(X,[1,1]).
\end{array}$$
So, $O$ has $[1,1]$ as neutral element. Therefore,
since,  by hypotheses, $O$ is associative, and by definition of interval-valued overlaps it continuous and positive, then
$O$ is a continuous and positive interval-valued t-norm.
\end{proof}

Observe that any interval-valued overlap function that is also an interval-valued t-norm satisfies the condition
$O([1,1],X)=X$. But there are interval-valued overlap functions, which are not associative (and which are not hence
a t-norm), that satisfy this property. For instance, take:
$$O(X,Y)= (Y \sqrt{X})\wedge (X\sqrt{Y})$$
or
$$O(X,Y)=(X \wedge Y) (X^2 \vee Y^2).$$

\subsection{Representable interval-valued overlap functions}

In this section, we study the representation of interval-valued overlap functions.

\begin{theorem}\label{teo-O1O2}
 Let $G_{O_1}$ and $G_{O_2}$ be overlap functions such that $G_{O_1}\leq _{O} G_{O_2}$. Then the function $\widetilde{G_{O_1}G_{O_2}}:L([0,1])^2\rightarrow L([0,1])$
defined by

\begin{equation}\label{eq-ov-int-ov}
\widetilde{G_{O_1}G_{O_2}}(X,Y)=[G_{O_1}(\underline{X},\underline{Y}),G_{O_2}(\overline{X},\overline{Y})]
\end{equation}
is an interval-valued overlap function.
\end{theorem}
\begin{proof}
It is immediate.
\end{proof}

Analogously to the notion of t-representability in \cite{CDK06},  an interval-valued overlap function $O$ is said to be $o$-representable if there exist overlap functions
 $G_{O_1}$ and $G_{O_2}$ such that $O=\widetilde{G_{O_1}G_{O_2}}$.   $G_{O_1}$ and $G_{O_2}$  are called representatives of $O$.

 Observe, however, that not every interval-valued overlap function is $o$-representable, for example
\[O(X,Y)=[\underline{X}\overline{X}\underline{Y}\overline{Y},\overline{X}\;\overline{Y}]\] is clearly not $o$-representable.

Now we intend to characteriza those interval-valued overlap functions which are $o$-representable. We start with the next definition. 

\begin{definition}
 Let $F:L([0,1])^n\rightarrow L([0,1])$ be a monotonic function. The left and right projections of $F$ are the functions $\underline{F},\overline{F}:[0,1]^n\rightarrow [0,1]$
defined by
\begin{equation}\label{eq-proj-F}
\aligned &  \underline{F}(x_1,\ldots,x_n){=}\underline{F([x_1,x_1],\ldots,[x_n,x_n])} \\
& \overline{F}(x_1,\ldots,x_n){=}\overline{F([x_1,x_1],\ldots,[x_n,x_n])},
\endaligned
\end{equation} respectively.
\end{definition}

\begin{proposition}\label{pro-O-proj}
 Let $O:L([0,1])^2\rightarrow L([0,1])$  be an interval-valued overlap function. If $O$ is strongly positive (SP), that is, it holds that either $\underline{X} = 0$ or  $\underline{Y} = 0$ whereas $O(X,Y)= [0,z]$, for some $z \in (0,1]$, then $\underline{O}$ as well as $\overline{O}$ are overlap functions.
\end{proposition}
\begin{proof}
 For $i=1,3,4$, since  $O$ satisfies the Property $(Oi)$, then, from Equation (\ref{eq-proj-F}), $\underline{O}$ clearly satisfies the Property $(Oi)$.
Note that $\underline{O}$ is continuous, since it is a composition of two continuous functions -- the interval-valued overlap $O$ and the left projection. So, $\underline{O}$, and analogously $\overline{O}$), are overlap functions.
It remains to prove that $\underline{O}$ satisfies $(GO2)$. In fact, if $\underline{O}(x,y)=0$ then one has that $O([x,x],[y,y]) = [0,z]$, for some $z \in [0,1]$.  If $z =0$ then, by (O2), it holds that either $[x,x]=[0,0]$ or  $[y,y]=[0,0]$, which implies either  $x = 0$ or $y=0$. If $z > 0$, by (SP), we have that $x = 0$ or $y=0$. On the other hand, if $x > 0$ and $y > 0$ then it holds that $O([x,x],[y,y]) > [0,0]$, which implies $\underline{O}(x,y) \geq 0$. But, when
$\underline{O}(x,y) = 0$, considering what it was proved above, we have that either $x=0$ or $y=0$, which is a contradiction, since $x > 0$ and $y > 0$. Therefore, $\underline{O}(x,y) > 0$ and hence it satisfies $(GO2)$.
\end{proof}

\begin{remark}
It is important to point out that there exist some interval overlap functions $O$ such that $\underline{O}$ is not an overlap function. For instance, if
$T_L$ is the \L ukasiewicz's t-norm then $O([a,b],[c,d])=[T_L(a,c),\min(b,d)]$ is an interval overlap function and, then, one has that  $\underline{O} = T_L$, which is not an overlap function, since it is not positive. This fact shows that the Property (SP) is a necessary condition for the Proposition \ref{pro-O-proj}.
\end{remark}

Now we can characterize $o$-representable interval-valued overlap functions.

\begin{theorem}\label{teo-O-rep-proj}
An interval-valued overlap function $O$ is $o$-representable if and only if  $O=\widetilde{\underline{O}\overline{O}}$.
\end{theorem}
\begin{proof} ($\Rightarrow$)
 If $O$ is $o$-representable then, by Theorem \ref{teo-O1O2}, it holds that  \begin{equation}\label{eq_asterisco} O=\widetilde{G_{O_1}G_{O_2}},\end{equation} for some overlaps $G_{O_1}$ and $G_{O_2}$ such that
$G_{O_1}\leq _{O} G_{O_2}$. Thus,
$$\begin{array}{lll}
   G_{O_1}(x,y) & = \underline{[G_{O_1}(x,y),G_{O_2}(x,y)]} & \mbox{ by Eq. (\ref{eq-under-over})} \\
& = \underline{\widetilde{G_{O_1}G_{O_2}}([x,x],[y,y])} & \mbox{ by Eq. (\ref{eq-ov-int-ov})} \\
& =  \underline{O([x,x],[y,y])} & \mbox{ by Eq. (\ref{eq_asterisco})} \\
& = \underline{O}(x,y) & \mbox{ by Eq. (\ref{eq-proj-F})} \\
  \end{array}$$
Analogously, one shows that $G_{O_2}=\overline{O}$. So, by Eq. (\ref{eq_asterisco}), it holds that $O=\widetilde{\underline{O},\overline{O}}$.

($\Leftarrow$) It is straightforward, following  from Proposition \ref{pro-O-proj}.
\end{proof}

In fact, we can go one step further, thanks to the following Lemma.

\begin{lemma}\label{lem-rep-inclusion}
 Let $F:L([0,1])^n\rightarrow L([0,1])$ be a monotonic function. Then, for any $X_i\in L([0,1])$,   with $i=1,\ldots,n$, it holds that
$$F(X_1,\ldots,X_n)=[\underline{F}(\underline{X_1},\ldots,\underline{X_n}),\overline{F}(\overline{X_1},\ldots,\overline{X_n})]$$
 if and only if
$F$ is inclusion monotonic.
\end{lemma}
\begin{proof}
 It is an easy and trivial extension for n-ary interval-valued functions of in \cite[Theorem 8]{DBSR11}.
\end{proof}

\begin{theorem}\label{teo-char-o-rep}
 An interval-valued overlap function $O$ is $o$-representable if and only if is inclusion monotonic.
\end{theorem}
\begin{proof} It is straightforward, following  from Lemma \ref{lem-rep-inclusion}, Proposition \ref{pro-O-proj}, and  theorems \ref{teo-O1O2} and  \ref{teo-O-rep-proj}.
%
%
\end{proof}

In \cite{Des08},
the notion of t-representability was introduced and led to the concepts of pseudo t-representability (denoted by $\mathcal{T}_T$), generalized pseudo t-representability (denoted by $\mathcal{T}_{T,t}$)
and a third  type without a particular name, which is denoted by $\mathcal{T}'_T$. {It is clear that whenever we substitute the t-norm $T$ by an overlap $O$, then $\mathcal{T}_O$ are
 $o$-representable interval-valued overlap functions,   whereas  $\mathcal{T}_{O,t}$ and $\mathcal{T}'_O$  are not. However, in the present paper,  we
provide a more generic class of ``representable'' interval-valued overlap functions,  which properly contains those three classes.}

\begin{definition}
 An interval-valued overlap $O$ is  semi $o$-representable if there exist  overlap functions $G_{O_i}:[0,1]^2$ $\rightarrow [0,1]$, with $i=1,\ldots,8$,
and functions
 $M_1,M_2:[0,1]^4\rightarrow [0,1]$, such that for each $X,Y\in L([0,1])$, it holds that:

\footnotesize
\begin{equation} \label{eq-semi-o-rep}
\begin{array}{ll}
 O(X,Y) {=} & [M_1(G_{O_1}(\underline{X},\overline{Y}),G_{O_2}(\overline{X},\underline{Y}),G_{O_3}(\underline{X}, \underline{Y}), G_{O_4}(\overline{X},\overline{Y})), \\  
 & \,\,M_2(G_{O_5}(\underline{X},\overline{Y}),G_{O_6}(\overline{X},\underline{Y}),  G_{O_7}(\underline{X},\underline{Y}),G_{O_8}(\overline{X},\overline{Y}))]
\end{array}
\end{equation}
\normalsize
\end{definition}

\begin{theorem} \label{teo-semi-o-rep}
Let $G_{O_i}$, with $i=1,\ldots,8$, be  overlap functions and   $M_1,M_2:[0,1]^4\rightarrow [0,1]$ be  aggregation functions such that
$G_{O_i}\leq_\mathfrak{O} G_{O_{i+4}}$, for $i=1,\ldots,4$, it holds that
$M_1\leq M_2$. Then, for the interval-valued overlap function
$O:L([0,1])^2\rightarrow L([0,1])$, defined as in Equation (\ref{eq-semi-o-rep}), it holds that:
\begin{enumerate}
 \item If $G_{O_1}=G_{O_2}$,  $G_{O_5}=G_{O_6}$ and $M_1$ and $M_2$ are commutative in the two first components then $O$ satisfies ($O1$);

\item If  either $M_1$ or $M_2$ satisfies the property $(M3)$ with respect to the fourth component, i.e., $M_1(x_1,x_2,x_3,x_4)$ $=0$ then $x_4=0$, and the same for $M_2$.  Then $O$ satisfies ($O2$);

\item If  either  $M_1$ or $M_2$ satisfies the property $(M4)$ with respect to the third component, then $O$ satisfies ($O3$);

\item If $G_{O_i}\leq_\mathfrak{O} G_{O_{i+4}}$, for $i=1,\ldots,4$, then $O$ satisfies ($O4$);

\item $O$ satisfies $(O5)$ if and only if $M_1$ and $M_2$ are continuous.
\end{enumerate}

\end{theorem}

\begin{proof} \,

\begin{enumerate}
\item We start proving that $(O1)$ holds. Take $X, Y \in L([0,1])$. Then:

\begin{eqnarray*}
\lefteqn{ O(X,Y)} && \\
&=&  [M_1(G_{O_1}(\underline{X},\overline{Y}),G_{O_1}(\overline{X},\underline{Y}),G_{O_3}(\underline{X},\underline{Y}),G_{O_4}(\overline{X},\overline{Y})),  \\
&& \; \; M_2(G_{O_5}(\underline{X},\overline{Y}),G_{O_5}(\overline{X},\underline{Y}),G_{O_7}(\underline{X},\underline{Y}),G_{O_8}(\overline{X},\overline{Y}))]  \\
&=& [M_1(G_{O_1}(\overline{Y},\underline{X}),G_{O_1}(\underline{Y},\overline{X}),G_{O_3}(\underline{Y},\underline{X}),G_{O_4}(\overline{Y},\overline{X})), \\
&& \; \; M_2(G_{O_5}(\overline{Y},\underline{X}),G_{O_5}(\underline{Y},\overline{X}),G_{O_7}(\underline{Y},\underline{X}),G_{O_8}(\overline{Y},\overline{X}))]  \\
&=&  [M_1(G_{O_1}(\underline{Y},\overline{X}),G_{O_1}(\overline{Y},\underline{X}),G_{O_3}(\underline{Y},\underline{X}),G_{O_4}(\overline{Y},\overline{X})),   \\
&& \;  \; M_2(G_{O_5}(G_{O_5}(\underline{Y},\overline{X}),\overline{Y},\underline{X}),G_{O_7}(\underline{Y},\underline{X}),G_{O_8}(\overline{Y},\overline{X}))]   \\ 
&=& O(Y,X).
 \end{eqnarray*}
%
%
\item Assume that $M_1$ or $M_2$ satisfy the property $(M3)$ with respect to the fourth component.  If $O(X,Y)=[0,0]$ then one has that:
 $$M_1(G_{O_1}(\underline{X},\overline{Y}),G_{O_2}(\overline{X},\underline{Y}),G_{O_3}(\underline{X},
 \underline{Y}),G_{O_4}(\overline{X},\overline{Y}))=0$$ 
 and
$$M_2(G_{O_5}(\underline{X},\overline{Y}),G_{O_6}(\overline{X},\underline{Y}),G_{O_7}(\underline{X},
\underline{Y}),G_{O_8}(\overline{X},\overline{Y}))=0$$
Since, as $M_1$ or $M_2$ satisfies (M3) with respect to the fourth component, then  it follows that either $G_{O_4}(\overline{X},\overline{Y})=0$ or  $G_{O_8}(\overline{X},\overline{Y})=0$.
In both cases,  by (O2), one has that either $\overline{X}=0$ or $\overline{Y}=0$, and, therefore, either $X=[0,0]$ or $Y=[0,0]$.  %
On the other hand,  if $XY=[0,0]$,  then one has that $X=[0,0]$ or $Y=[0,0]$. Assume that $X=[0,0]$. It follows that:
\begin{eqnarray*}
\lefteqn{O([0,0],Y)}&& \\ &=&[M_1(G_{O_1}(0,\overline{Y}),G_{O_2}(0,\underline{Y}),G_{O_3}(0,\underline{Y}),G_{O_4}(0,\overline{Y})), \\
&& M_2(G_{O_5}(0,\overline{Y}),G_{O_6}(0,\underline{Y}),G_{O_7}(0,\underline{Y}),G_{O_8}(0,\overline{Y}))]  \\
&& \mbox{by Eq. (\ref{eq-semi-o-rep})} \\
& = &[M_1(0,0,0,0),M_2(0,0,0,0)] \mbox{ by (O2)} \\
& = & [0,0] \mbox{ by (A2)}
 \end{eqnarray*}
If $Y=[0,0]$ we can make an analogous calculation. Therefore, if $XY=[0,0]$ then it holds that $O(X,Y)=[0,0]$. Therefore, the Property $(O2)$ holds.

\item It can be done analogously to item $2$;

\end{enumerate} The items $4.$ and $5.$ hold straightforwardly, since the composition of monotonic and continuous functions is monotonic and continuous.
\end{proof}

\begin{example}
A trivial example of aggregation functions $M_1$ and $M_2$, and overlap functions $G_{O_i}$, with $i=1,\ldots,8$,
satisfying the conditions of the Theorem \ref{teo-semi-o-rep} can be obtained taking $G_{O_i}$ such that  $G_{O_i}=G_{O_j}$, for each  $i,j\in \{1,\ldots,8\}$, and considering the aggregation functions
$M_1(x_1,\ldots,x_4)=x_4$ and $M_2(x_1,\ldots,x_4)=x_3$. \end{example}


\begin{corollary}
If  there exist functions $M_1$, $M_2$ and $G_{O_i}$, with $i=1,\ldots,8$,
satisfying the condition of Theorem  \ref{teo-semi-o-rep}, then the interval-valued overlap function $O$ defined as in Equation (\ref{eq-semi-o-rep})
is  semi $o$-representable.
\end{corollary}

Semi $o$-representability and representability are related, as we discuss next.

\begin{proposition}
 Let $O$ be an interval-valued overlap.
If  $O$ is $o$-representable or there exists a t-norm $T$ such that either $O=\mathcal{T}_T$ or $O=\mathcal{T}_{T,t}$, for some $t\in [0,1]$, or
  $O=\mathcal{T}'_T$,  then $O$ is semi $o$-representable.
\end{proposition}
\begin{proof}
For the first case, that is, if $O=\widetilde{G_O'G_O}$, for some overlaps functions $G_O'$ and $G_O$, then it is sufficient to consider $G_{O_3}=G_O'$, $G_{O_8}=G_O$,  $M_1(x_1,\ldots,x_4)=x_3$ and $M_2(x_1,\ldots,x_4)=x_4$.
 For the second case, that is, if  $O=\mathcal{T}_T$, then it is sufficient to consider $O_i=T$, for  $i=3,5,6$, $M_1(x_1,\ldots,x_4)=x_3$ and $M_2(x_1,\ldots,x_4)=\max(x_1,x_2)$.
For the third case, that is, if $O=\mathcal{T}_{T,t}$, then it is sufficient consider $G_{O_i}=T$, for  $i=3,5,6,8$,  $M_1(x_1,\ldots,x_4)=x_3$ and 
$M_2(x_1,\ldots,x_4)=\max(x_1,x_2,$ $T(t,x_4))$.
For the fourth case, that is,  $O=\mathcal{T}'_T$, then it is sufficient to consider $G_{O_i}=T$, for  $i=1,2,8$, $M_1(x_1,\ldots,x_4)=\min(x_1,x_2)$ and $M_2(x_1,\ldots,x_4)=x_4$.
\end{proof}

Let $F:L([0,1])\rightarrow L([0,1])$ be defined by $F(X)=m(X)+\frac{1}{2}(X-m(X))$, where $m(X)$ is the middle point of $X$, that is $m(X)=\frac{\underline{X}+\overline{X}}{2}$. \cite{Moo79}
An alternative definition of $F$ is the following:
$$F(X)=\left [\frac{\underline{X}+m(X)}{2},\frac{\overline{X}+m(X)}{2}\right ].$$

The function $O:L([0,1])^2\rightarrow L([0,1])$ defined by
$$O(X,Y)=F(X)\wedge F(Y)$$
 is an interval-valued overlap (observe that $O$   is Moore continuous because $F$ and the infimum are
 Moore continuous  \cite{AB97,SBA06}).
 Notice that $F$ is  not inclusion monotone \cite{AB97,Moo79} and, therefore, by Theorem \ref{teo-char-o-rep}, $O$
is not $o$-representable. Moreover, by a simple calculation, one may show that
\begin{equation}\label{eq-O-F}
 O(X,Y)=[\min(a\underline{X}+b\overline{X},a\underline{Y}+b\overline{Y}),\min(b\underline{X}+a\overline{X},b\underline{Y}+a\overline{Y})]
\end{equation}
where {$a=\frac{3}{4}$} and $b=\frac{1}{4}$. Thus, if there exists $M_1: [0,1]^4\rightarrow [0,1]$ and overlap functions $G_{O_i}$, with $i=1,\ldots,4$, such that

\begin{equation}\label{eq-char-F-O}
 M_1(G_{O_1}(\underline{X},\overline{Y}),G_{O_2}(\overline{X},\underline{Y}),G_{O_3}(\underline{X},\underline{Y}),G_{O_4}(\overline{X},
 \overline{Y}))= \min(a\underline{X}+b\overline{X},a\underline{Y}+b\overline{Y})
\end{equation}
then, without loss of generality, we can think that $G_{O_i}(x,y)$ $=x$ in some condition and   $G_{O_i}(x,y)=y$ in the other cases, in such a way that, for each $X$ and $Y$, it holds that
$\{G_{O_1}(\underline{X},\overline{Y}),G_{O_2}(\overline{X},\underline{Y}),G_{O_3}(\underline{X},\underline{Y}),
G_{O_4}(\overline{X},\overline{Y})\}=\{\underline{X},\overline{X},\underline{Y},\overline{Y}\}$. So, by ($O1$), one has that $G_{O_1}(x,y)=x$ if and only if $G_{O_2}(x,y)=y$.
Thus, for the particular case where $X=[0,0]$, then it holds that $G_{O_1}(0,0.5)=0$ if and only if $G_{O_2}(0,0.5)=0.5$,  which is in contradiction with the condition
($O2$). Therefore, $\underline{O(X,Y)}$ for $O$ defined as in Eq. (\ref{eq-O-F}) can not be obtained as in Eq. (\ref{eq-char-F-O}).


\subsection{Migrative  and Homogeneous Interval-valued Overlap Functions}

In this subsection we make a study analogous to  \cite[Proposition 1]{Bus10a}.\footnote{Notice that the item (iii) of
 Proposition 1 in \cite{Bus10a} is wrong. To see that, just consider the constant function $G_O(x,y)=0$.}

\begin{proposition}\label{pro-O-hom-mig}
 For an interval-valued function $O:L([0,1])^2\rightarrow L([0,1])$, it holds that:

\begin{enumerate}
 \item If $F$ satisfies the Property $(O2)$  then $O$  does not satisfy the self-duality property (SDP), that is,  the equation
\begin{equation}\label{eq-SDP}
 F(X,Y)=F(X^c,Y^c)^c
\end{equation}
fails,  for some $X,Y\in L([0,1])$;

\item If $F$ is migrative then $F$ satisfies the Property $(O1)$;
\item If $F$ is homogeneous of order $K$ then $F([0,0],[0,0])=[0,0]$;
\item If $F$ is homogeneous of order $[1,1]$ and $F([1,1],[1,1])$ $=[1,1]$ then $F$ is idempotent;
\item If $F$ is migrative and idempotent then $F$ is also homogeneous of order $[1,1]$;
\item If $F$ is migrative and has $[1,1]$ as neutral element then $F$ is homogeneous of order $[2,2]$.
\end{enumerate}
\end{proposition}
\begin{proof} \,
 \begin{enumerate}
  \item By the Property $(O2)$, we have that $F([0,0],$ $[1,1])=[0,0]$. On the other hand, by the definition of complement and the Property  $(O2)$,
we have that $F([0,0]^c,[1,1]^c)^c=F([1,1],[0,0])^c=[0,0]^c=[1,1]$. Therefore, the Equation (\ref{eq-SDP})  fails.

\item If $F$ is migrative then $F(X,Y)=F([1,1]X,Y)=F([1,1],XY)=F([1,1],YX)=F(Y,X).$


\item It follows that $F([0,0],[0,0])=[0,0]^KF(X,Y)=[0,0]$.

\item It holds that $F(X,X)=X^{[1,1]}F([1,1],[1,1])=X[1,1]=X$

\item By Proposition \ref{pro-IV-mig1}, one has that  $F( X, Y)=g_{F}(XY)$. Since $F$ is idempotent, then $g_{F}(X)=
g_{F}(\sqrt{X}\sqrt{X})=F(\sqrt{X},\sqrt{X})=\sqrt{X}$. Therefore, it holds that \\
$F(\alpha X,\alpha Y)=g_{F}(\alpha^2 XY)=\alpha \sqrt{XY}=\alpha g_{F}(XY)=\alpha F(X,Y).$

\item By Lemma \ref{lem-IV-mig1} and because $[1,1]$ is neutral element,  one has that $F(X,Y)=F([1,1],XY)=XY$. So, it follows that 
$F(\alpha X,\alpha Y)=\alpha^2 XY=\alpha^2F(X,Y).$
 \end{enumerate}
\end{proof}

\begin{corollary}
Let $F: L([0,1])^2\rightarrow L([0,1])$ be an interval-valued migrative function. $F$ is homogeneous of order $[1,1]$ if and only if $F$ is idempotent.
\end{corollary}

      \begin{proof}
       It is straightforward, following from  items 4 and 5 of Proposition \ref{pro-O-hom-mig}.
      \end{proof}

\begin{theorem}\label{theo-O-G-Mig}
  $O:L([0,1])^2\rightarrow L([0,1])$ is an interval-valued migrative overlap function if and only if there  exists a monotonic and Moore continuous interval
function $g_{O}:L([0,1])\rightarrow L([0,1])$ such that, for any $X,Y\in L([0,1])$, it holds that $O(X,Y)=g_{O}(XY)$, $g_{O}([0,0])=[0,0]$,
$g_{O}([1,1])=[1,1]$ and $g(X)\not\in \{[0,0],[1,1]\}$ when $X\not\in \{[0,0],[1,1]\}$.
\end{theorem}
\begin{proof}
($\Rightarrow$) By Proposition \ref{pro-IV-mig1}, the function $g_{O}(X)=O([1,1],X)$ is such that
 $O(X,Y)=g_{O}(XY)$, for any $X,Y\in L([0,1])$. Since $O$ is monotonic and Moore continuous, then obviously $g_{O}$ also is.
On the other hand, by the Property ($O2$), one has that
$g_{O}(X)=[0,0]$ if and only if $O(X,[1,1])=[0,0]$ if and only if $X=[0,0]$. Analogously, it  is possible to prove that
$g_{O}(X)=[1,1]$ if and only if  $X=[1,1]$.

($\Leftarrow$) By Proposition \ref{pro-IV-mig1}, the function $O(X,Y)=g_{O}(XY)$ is migrative,  and since $g_{O}$ is
increasing and Moore continuous, so is $O$. The symmetry follows from the commutativity of the product. The Properties ($O2$) and  ($O3$) are straightforward, following  from
similar properties of $g_{O}$. So, $O$ is a migrative interval-valued overlap function.
\end{proof}

\begin{theorem}
 Let $O:L([0,1])^2\rightarrow L([0,1])$ be an  interval-valued overlap function. If $O$ is migrative then it is $o$-representable.
\end{theorem}
\begin{proof} Let $g_O$ be the Moore-continuous function of Theorem \ref{theo-O-G-Mig}. Since $g_O$ is Moore-continuous,  then there exist two continuous functions
$g_1,g_2:[0,1]\rightarrow [0,1]$ such that $g_O(X)=[g_1(\underline{X}),g_2(\overline{X})]$ (see p.52 of \cite{Moo79}). Since, $g_O$ is monotonic then clearly $g_1$ and $g2$ are increasing
and, since it holds that $g_O([0,0])=[0,0]$ and $g_O([1,1])=[1,1]$,  then $g_1(0)=g_2(0)=0$ and $g_1(1)=g_2(1)=1$. Therefore, by   \cite[Theorem 9]{Bus10a}, the functions
$G_{O_1},G_{O_2}:[0,1]^2\rightarrow [0,1]$, defined by $G_{O_1}(x,y)=g_1(xy)$ and $G_{O_2}(x,y)=g_2(xy)$, respectively,  are migrative overlaps functions.  We prove that $O=\widetilde{G_{O_1}G_{O_2}}$. It follows that:

$\begin{array}{ll}
  O(X,Y) & =g_O(XY)=g_O([\underline{X}\underline{Y},\overline{X}\overline{Y}]=
[g_1(\underline{X}\underline{Y}),g_2(\overline{X}\overline{Y})] \\
& =[G_{O_1}(\underline{X},\underline{Y}),G_{O_2}(\overline{X},\overline{Y})]=\widetilde{G_{O_1}G_{O_2}}(X,Y)
 \end{array}$.
 
\end{proof}

\begin{proposition} \label{pro-hom-rep}

Let $O:L([0,1])^2 \rightarrow L([0,1])$ be an o-representable overlap function. $O$ is homogeneous of order $K=[k_1,k_2]$, with $0<k_1\leq k_2$, if and only if their representatives are homogeneous overlaps functions of orders $k_2$ and $k_1$, respectively.
\end{proposition}
\begin{proof}
 Since  $O:L([0,1])^2\rightarrow L([0,1])$ is an $o$-representable interval-valued overlap function then by Theorem \ref{teo-O1O2} and Lemma \ref{lem-rep-inclusion},
$O=\widetilde{\underline{O},\overline{O}}$ and, by Proposition \ref{pro-O-proj}, $\underline{O}$ is an overlap function.

($\Rightarrow$)  Consider $\alpha,x,y\in [0,1]$. It follows that:
$$\begin{array}{lll}
 \underline{O}(\alpha x,\alpha y) & = \underline{O([\alpha x,\alpha x],[\alpha y,\alpha y])} & \mbox{ by Eq. (\ref{eq-proj-F})} \\
 & = \underline{O([\alpha,\alpha][x,x],[\alpha,\alpha][y,y])} & \\
 & = \underline{[\alpha,\alpha]^KO([x,x],[y,y])} &  \mbox{ by homogeneity}\\
 & = \alpha^{k_2} \underline{O([x,x],[y,y])} & \\
& = \alpha^{k_2}\underline{O}(x,y) & \\
\end{array}$$
Therefore, $\underline{O}$ is an homogeneous overlap function  of order $k_2$. Analogously,  $\overline{O}$ is an homogeneous overlap function  of order $k_1$.

($\Leftarrow$) Consider $\alpha,X,Y\in L([0,1])$. It follows that:
 $$\begin{array}{lll}
   \widetilde{\underline{O},\overline{O}}(\alpha X,\alpha Y) & = [\underline{O}(\underline{\alpha} \underline{X},\underline{\alpha} \underline{Y}),
\overline{O}(\overline{\alpha} \overline{X},\overline{\alpha} \overline{Y})] &  \mbox{ by Eq. (\ref{eq-ov-int-ov})} \\
& = [\underline{\alpha}^{k_2}\underline{O}( \underline{X},\underline{Y}),
\overline{\alpha}^{k_1}\overline{O}(\overline{X}, \overline{Y})] & \\
& \;\;\; \mbox{ by homogeneity}\\
& = \alpha^K [\underline{O}( \underline{X},\underline{Y}),\overline{O}(\overline{X}, \overline{Y})] & \\
& = \alpha^K \widetilde{\underline{O}\overline{O}}(X,Y)  &  \mbox{ by Eq. (\ref{eq-ov-int-ov})} \\
   \end{array}$$

\end{proof}

{

As in the real case, the following result holds.

\begin{theorem}
 Consider $K=[k_1,k_2]$, where $0<k_1\leq k_2$. The unique interval-valued function $F:L([0,1])^2\rightarrow L([0,1])$ that is migrative, homogeneous of order $K$ and satisfies $F([1,1],[1,1])=[1,1]$ is
$$F(X,Y)=(XY)^{\frac{K}{[2,2]}}.$$
\end{theorem}
\begin{proof}
 For any $\alpha,X,Y\in L([0,1])$, by the homegeneity and Theorem \ref{theo-O-G-Mig}, one has that $F(\alpha X,\alpha Y)=\alpha^K F(X,Y)=\alpha^K g_F(XY)$. But, also by Theorem \ref{theo-O-G-Mig}, it holds that
\[F(\alpha X,\alpha Y)=g_F(\alpha X\alpha Y).\] So, it follows that $\alpha^K g_F(XY)=g_F(\alpha^2 XY)$. Thus, taking $X=Y=[1,1]$, we have that $g_F(\alpha^2)=\alpha^K$ and, so,
$g_F(\alpha)=\alpha^{\frac{K}{[2,2]}}$.
\end{proof}

\subsection{Interval-valued overlap functions and t-norms}

There are several non-equivalent but related interval-valued extensions of t-norms, as in \cite{BT06,DCK04,DBSR11,GWW96,Jen97}. Here we  consider the definition provided by \cite{DCK04}, which
is more general than the notion of \cite{BT06,Jen97} and consistent with the notion of lattice-valued t-norms as given,
for example, in \cite{BSC06,dBM99,Des08,Pal12}.

\begin{lemma}\label{lem-gO-idem}
If  $O:L([0,1])^2\rightarrow L([0,1])$  is an associative interval-valued overlap function then $g_O$ is idempoten and self-contractive, that is, it holds that
$g_O(g_O(X))=g_O(X)$ and
$g_O(X)\subseteq X$, respectively.
\end{lemma}
\begin{proof}
  Remember that $g_O:L([0,1])\rightarrow L([0,1])$ is defined by $g_O(X)=O(X,[1,1])$. Thus,
for any $X\in L([0,1])$, we have that $g_O(X)=O(X,[1,1])=O(X,O([1,1],[1,1]))=O(O(X,[1,1]),[1,1])=g_O(g_O(X))$, that is,  $g_O$ is idempotent. Now, clearly, since $O$ is monotonic, we have that
$g_O$ also is monotonic and so their projection $\underline{g_O}$ and$\underline{g_O}$. Analogously, from Moore continuity of $O$, we have that
$g_O$ is also Moore continuous and so their projection.  Finally, it is immediate that $\underline{g_O}(0)= \overline{g_O}(0)=0$ and
$\underline{g_O}(1)= \overline{g_O}(1)=1$. Then, for any $X\in L([0,1])$, there exist $x,y\in [0,1]$ such that
$\underline{g_O}(x)= \underline{X}$ and $\overline{g_O}(y)=\overline{X}$. Since $\underline{g_O}\leq \overline{g_O}$ then one has that $x\leq y$.
Denote $a=\overline{g_O}(x)$ and $b= \underline{g_O}(y)$. Thus, it holds that $g_O([x,x])=[\underline{g_O}(x),\overline{g_O}(x)]=[\underline{X},a]$
and $g_O([y,y])=[\underline{g_O}(y),\overline{g_O}(y)]=[b,\overline{X}]$. Since $g_O$ is monotonic, then $g_O([x,x])\leq_{Pr} g_O([y,y])$,
that is,  $[\underline{X},a]\leq_{Pr} [b,\overline{X}]$. Thus, one has that $a\leq \overline{X}$ and $\underline{X}\leq b$. Therefore, it follows that
$g_O([x,x]) \leq_{Pr} X \leq_{Pr} g_O([y,y])$,
and, so, by idempotency and isotonicity, it holds that $[\underline{X},a]\leq g_O(X)\leq [b,\overline{X}]$. Hence, it follows that
$\underline{X}\leq \underline{g_O(X)}\leq \overline{g_O(X)}\leq \overline{X}$, which implies that $g_O(X)\subseteq X$.
\end{proof}

Notice that the self-contractiveness  implies   in contractiveness  in the sense of \cite{BBF10}, and so self-contractiveness is stronger than contractiveness.

\begin{theorem}
 Let  $O:L([0,1])^2\rightarrow L([0,1])$  be an associative interval-valued overlap function such that $g_O$ is either surjective or inclusion monotonic. Then $O$ is an interval-valued t-norm.
\end{theorem}
\begin{proof}
 We prove that $[1,1]$ is a neutral element of $O$.  If $g_O$ is surjective then there exists $Y\in L([0,1])$ such that
\begin{equation}\label{eq12}
g_O(Y)=X
\end{equation}
and, therefore,
\begin{equation}\label{eq13}
g_O(g_O(Y))=g_O(X).
\end{equation}
Thus, by the idempotency of $g_O$ (Lemma \ref{lem-gO-idem}) and  by Eq. (\ref{eq13}), we have that $g_O(Y)=g_O(X)$. Therefore, by Eq. (\ref{eq12}), one has that $g_O(X)=X$ and, then,  $O(X,[1,1])=g_O(X)=X$, that is,  $[1,1]$ is a neutral element of $O$.
On the other hand, if $g_O$ is inclusion monotonic, then by Lemma \ref{lem-rep-inclusion}, one has that
\begin{equation} \label{eq14}
g_O(X)=[\underline{g_O}(\underline{X}),\overline{g_O}(\overline{X})].
\end{equation}
Therefore, since $[\underline{X},\underline{X}]\subseteq X$ and $[\underline{X},\underline{X}] \leq_{Pr} X$, then, by the inclusion monotonicity  and contractiveness  (Lemma \ref{lem-gO-idem}), it follows that
\begin{equation}\label{eq15}
g_O([\underline{X},\underline{X}])\subseteq g_O(X)\subseteq X
\end{equation}
 and
\begin{equation}\label{eq16}
g_O([\underline{X},\underline{X}]) \leq_{Pr} g_O(X).
\end{equation}
So, from Eq. (\ref{eq15}), it holds that $\underline{X}\leq \underline{g_O}(\underline{X})$,  and, from Eq. (\ref{eq16}), one has that
$\underline{g_O}(\underline{X})\leq \underline{X}$, that is,  $\underline{g_O}(\underline{X}) = \underline{X}$. Analogously, it is possible to prove that
 $\overline{g_O}(\underline{X}) = \overline{X}$. Therefore, from Eq. (\ref{eq14}), one has that  $O(X,[1,1])=g_O(X)=X$, that is,  $[1,1]$ is a neutral element of $O$.
\end{proof}

\section{Interval-valued OWA operators with interval-valued weighted vectors}

In this section we propose a definition that generalizes OWA operators to the interval-valued setting. In most of the cases, this generalization is carried out by considering interval-valued inputs, with, however,  pointwise weights. Our definition handles both interval-valued inputs and  weights. We start introducing the notion of weighing vector in our setting.

\begin{definition} Let $M:L([0,1])^n \to L([0,1])$ be an interval-valued aggregation function.
An $n$-tuple $(W_1,\dots,W_n) \in L([0,1])^n$ is said to be an $M$-weighted vector if and only if
\begin{equation} \label{eq-wv}
M(W_1,\dots,W_n)=[1,1] \; .
\end{equation}
\end{definition}

Note that this definition extends the usual definition of weighted vector in the real-valued case. However, since we consider a general interval-valued aggregation function $M$ for normalizing, we get more flexibility.

\begin{remark}
Notice that:
\begin{enumerate}
\item The vector $([1,1],\dots,[1,1])$ is a weighted vector for every interval-valued aggregation function $M$.
\item Consider the function $M:L([0,1])^n \to L([0,1])$, given by
\[
M(X_1,\dots,X_n)= [\max(\underline{X}_1,\dots,\underline{X}_n),\max(\overline{X}_1,\dots,\overline{X}_n)]
\]
Then $(W_1,\dots,W_n) \in L([0,1])^n$ is a $M$-weighted vector if and only if there exists $i_0 \in \{1,\dots,n\}$ such that $W_{i_0}=[1,1]$.

\item  Consider the function $M:L([0,1])^n \to L([0,1])$ given by
\[
M(X_1,\dots,X_n)= \left [\min \left (1,\sum _{i=1}^n \underline{X}_i)\right ),\min \left (1,\sum _{i=1}^n \overline{X}_i)\right)\right]
\]
Then $(W_1,\dots,W_n) \in L([0,1])^n$ is a $M$-weighted vector if and only if $\sum\limits_{i=1}\limits^n \underline{X}_i \ge 1$.
\end{enumerate}
\end{remark}

Our definition of weighing vector allows us to introduce the concept of interval-valued OWA operator considering interval-valued weights.

\begin{definition}\label{def:OWAIV}
Let $O:L([0,1])^2 \to L([0,1])$ be an interval-valued overlap function such that $O([1,1],X)=X$, for every $X \in L([0,1])$. Let $M:L([0,1])^n \to L([0,1])$ be an interval-valued aggregation function such that for every $X_1,\dots,X_n,Y  \in L([0,1])$, the identity
\begin{equation}\label{eqn:propOWA}
M(O(X_1,Y),\dots,O(X_n,Y))=O(M(X_1,\dots,X_n),Y)
\end{equation}
holds. Let $W=(W_1,\dots,W_n) \in L([0,1])^n$ be a $M$-weighted vector. Then, an interval-valued OWA operator of dimension $n$ is defined as a function $$IV{-}GOWA:L([0,1])^n \to L([0,1]),$$ given by
\[
IV{-}GOWA(X_1,\dots,X_n){=}M(O(W_1, X_{(1)}){,}\dots, O(W_n,X_{(n)})),
\]
where $(.)$ denotes a permutation of $\{1,\dots,n\}$ such that $X_{(n)} \leq X_{(n-1)} \leq \dots \leq X_{(1)}$ for an admissible order $\leq$.
\end{definition}

\begin{example}\label{ex-OM1} Let be $M(X_1,\dots,X_n)=(X_1 \cdot\dots\cdot X_n)^\frac{1}{n}$ and $O(X,Y)=XY$. Then we have that
\begin{eqnarray*}
M(O(X_1,Y),\dots,O(X_n,Y))&=&M(X_1Y,\dots,X_nY) \\
&=& (X_1 \cdot\dots\cdot X_n)^\frac{1}{n}Y,
\end{eqnarray*}
for every $X_1,\dots,X_n,Y \in L([0,1])$. Then, we are in the setting of Definition \ref{def:OWAIV}. Note that the only possible weighted vector in this setting is $W=([1,1],\dots,[1,1])$.
\end{example}

\begin{example} \label{ex-OM2}
Define
\[
M(X_1,\dots,X_n)=\begin{cases} [1,1] &\text{ if } \max(X_1,\dots,X_n)=[1,1]
\\
[0,0] &\text{ otherwise.}
\end{cases}
\]
Then, for every interval-valued overlap function $O$ and $X_1,\dots,X_n,Y \in L([0,1])$ it holds that:
\[
M(O(X_1,Y),\dots,O(X_n,Y))=O(M(X_1,\dots,X_n),Y).
\]
Note that, in this case, $(W_1,\dots, W_n)$ is a weighted vector if and only if the identity \linebreak $\max (W_1,\dots,W_n)=[1,1]$ holds (with respect to the admissible order $\le$).
\end{example}

The examples  \ref{ex-OM1} and \ref{ex-OM2} just present two pairs of interval-valued overlap functions and interval-valued
aggregation functions that
satisfy the Eq. (\ref{eqn:propOWA}) and, therefore, we can define an interval-valued OWA operator from those pairs.
Thus, it would be interesting to have a characterization of such pairs. This fact motivate the following question: ``How can we characterize the interval-valued overlap functions that satisfy the  Equation (\ref{eqn:propOWA}), for some interval-valued aggregation function?''  We answer this question in the following proposition.

\begin{proposition}
Consider $O(X,Y)=XY$ and  let $M:L([0,1])^n \to L([0,1])$ be an interval-valued aggregation function. Then the Identity (\ref{eqn:propOWA}) holds if and only if $M$ is an interval-valued homogeneous function of order $[1,1]$.\footnote{A similar result was presented in \cite{LBH15} for interval-valued
t-norms and t-conorms.}
\end{proposition}
\begin{proof}
 Firstly, suppose that $M$ is (interval-valued) homogeneous of order $[1,1]$. Since $O(X,Y)=XY$, then
 $$
 \begin{array}{rcl}
 M(O(X_1,Y),\dots,O(X_n,Y)) & = & M(X_1 Y, \ldots, X_n Y)\\
                                                & = & YM(X_1, \ldots, X_n)\\
                                                & = & O(M(X_1, \ldots, X_n),Y)
\end{array}
 $$
 Conversely, if the Identity (\ref{eqn:propOWA}) holds, it follows that
 $$
 \begin{array}{rcl}
 M(\alpha X_1, \ldots, \alpha X_n) & = & M(O(\alpha,X_1), \ldots, O(\alpha,X_n)) \\
                                                      & = & O(\alpha, M(X_1, \ldots, X_n)) \\
                                                      & = & \alpha M(X_1, \ldots, X_n)
 \end{array}
 $$
 Therefore, $M$ is homogeneous of order $[1,1]$.
\end{proof}

\subsection{Some properties of IV{-}GOWA operators}

In the following results, we show how some of the most important properties demanded to real-valued OWA are also fulfilled by IV-GOWA operators.

\begin{proposition} \label{idem}
Let $O$ and $M$ be as in Definition \ref{def:OWAIV}. Then, for any $M$-weighted vector $W \in L([0,1])^n$, it follows that:
\[
IV{-}GOWA(X,\dots,X)=X,
\]
that is, $IV{-}GOWA$ is idempotent.
\end{proposition}

\begin{proof} It follows  from the properties demanded to $M$, $O$ and $W$. \end{proof}

\begin{proposition}
Let $M$ and $O$ be as in Definition \ref{def:OWAIV}, and $W \in L([0,1])^n$ be an $M$-weighted vector. Then, the function $IV-GOWA$ defined in terms of $M$, $O$ and $W$ is an interval-valued aggregation function with respect to $\le$.
\end{proposition}
\begin{proof} First of all, note that, by Proposition \ref{idem}, it follows that $$IV{-}GOWA( [0,0],\dots,[0,0] ) = [0,0]$$ and $$IV{-}GOWA([1,1],\dots,[1,1]) = [1,1]$$
Moreover, the monotonicity of $IV{-}GOWA$ follows from the monotonicity with respect to $\le$ of both $M$ and $O$.
\end{proof}

\begin{proposition}
Consider $i_0 \in \{1,\dots,n\}$. Then, in the setting of Definition \ref{def:OWAIV}, assume that $M([0,0],\dots,$ $[0,0],X,[0,0],\dots,[0,0])=X$. Consider the vector $W$ 
defined by:
\[
W_i=\begin{cases} [1,1] &\text{ if } i=i_0 \ ;
\\
[0,0] &\text{ otherwise.}
\end{cases}
\]
Then it holds that:
\begin{enumerate}
\item $W$ is an $M$-weighted vector.
\item $IV{-}GOWA (X_1,\dots,X_n)=X_{(i_0)}$, where $X_{(i)}$ denotes the $i$-th largest interval  $X_1,\dots,X_n$,  with respect to an admissible total order.
\end{enumerate}
\end{proposition}
\begin{proof}
It is immediate, following from the Definition \ref{def:OWAIV}. \end{proof}

\begin{proposition}
Consider  
$$M(X_1,\dots,X_n)=\left [\min \left (1,\sum\limits_{i=1}\limits^n \underline{X} _i \right ), \min \left  (1,\sum\limits_{i=1}\limits^n \overline{X} _i \right) \right]$$
and $O(X,Y)=XY$. Then it holds that:
\begin{enumerate}
\item $W=([\frac{1}{n},\frac{1}{n}],\dots,[\frac{1}{n},\frac{1}{n}])$ is a $M$-weighted vector.
\item The corresponding $IV{-}GOWA$ operator is given by:
\[
IV{-}GOWA(X_1,\dots,X_n)=\left [\frac{1}{n} \sum\limits_{i=1}\limits^n \underline{X}_i,\frac{1}{n} \sum\limits_{i=1}\limits^n \overline{X}_i \right]
\]
\end{enumerate}
\end{proposition}
     \begin{proof}
     It is immediate.
      \end{proof}


\section{Conclusions and future research}

We have presented a study about interval-valued overlaps functions, showing conditions to ensure their  representability and  discussing the important properties of   migrativity and
homogeneity. We have also introduced the notion of semi $o$-representability. We have also discussed a method to build interval-valued OWA operators when considering interval-valued weighting vectors.

As a further work,  we will investigate ordinal sums and  additive generators of interval-valued overlap functions, in the line of what was done in
\cite{DB14,DB15}, aiming at practical applications. We also will study  how to characterize the homogenous interval-valued aggregation functions.


\section*{Acknowledgments}  This work was partially supported  by the Spanish Ministry of Science and Technology under the project  TIN2013-40765-P and   by the Brazilian 
funding agency CNPQ under Processes  307681/2012-2, 306970/2013-9, 406503/2013-3, 481283/2013-7 and 232827/2014-1.


\end{document}